\documentclass{article}
\usepackage{amsmath}
\usepackage{amsthm}
\usepackage[multiple]{footmisc}
\usepackage{geometry}
\usepackage{graphicx} 
\usepackage{booktabs}
\usepackage{tikz}
\usepackage{natbib}
\usepackage[most]{tcolorbox} 
\usepackage{xcolor} 
\usepackage{geometry} 
\usepackage{subcaption}
\usepackage{hyperref}
\usepackage{threeparttable}
\usepackage{chngcntr}

\usetikzlibrary{positioning, arrows.meta}
\newtheorem{theorem}{Theorem}[section]
\newtheorem{definition}{Definition}[section]

\definecolor{myblack}{RGB}{0, 0, 0}
\definecolor{titlebackground}{RGB}{0, 0, 0}

\newcommand{\user}[1]{\textcolor{myblack}{\textbf{User:} #1}\par}
\newcounter{promptcounter}

\newenvironment{prompt}[1][]
  {\refstepcounter{promptcounter}
   \begin{tcolorbox}[
    colframe=myblack,
    colback=white,
    coltitle=white,
    colbacktitle=titlebackground,
    fonttitle=\bfseries,
    title=Prompt \thepromptcounter: #1,
    enhanced,
    arc=0mm,
    attach boxed title to top left={
      xshift=1cm,
      yshift*=-\tcboxedtitleheight/2
    },
    boxed title style={
      arc=0mm,
      colframe=titlebackground
    },
    top=12pt,
    before upper={\parskip10pt}
   ]
  }
  {\end{tcolorbox}}

\title{The Challenge of Using LLMs to Simulate Human Behavior:\\ A Causal Inference Perspective}
\author{George Gui and Olivier Toubia\thanks{George Gui (zg2467@gsb.columbia.edu) is an Assistant Professor of Marketing and Olivier Toubia (ot2107@gsb.columbia.edu) is the Glaubinger Professor of Business at Columbia Business School. The authors would like to thank Jia Li and Angela Qianya Wang for their research assistance.}}
\date{November 22, 2025}
\begin{document}

\maketitle

\begin{abstract}
Large Language Models (LLMs) have shown impressive potential to simulate human behavior. We identify a fundamental challenge in using them to simulate experiments: when LLM-simulated subjects are blind to the experimental design (as is standard practice with human subjects), variations in treatment systematically affect unspecified variables that should remain constant, violating the unconfoundedness assumption. Using demand estimation as a context and an actual experiment with 40 different products as a benchmark, we show this can lead to implausible results. While confounding may in principle be addressed by controlling for covariates, this can compromise ecological validity in the context of LLM simulations: controlled covariates become artificially salient in the simulated decision process. 
We show formally that confoundness stems from ambiguous prompting strategies. Therefore, it can be addressed by developing unambiguous prompting strategies through unblinding, i.e., revealing the experiment design in LLM simulations. Our empirical results show that this strategy consistently enhances model performance across all tested models, including both out-of-box reasoning and non-reasoning models. We also show that it is a technique that complements fine-tuning: while fine-tuning can improve simulation performance, an unambiguous prompting strategy makes the predictions robust to the inclusion of irrelevant data in the fine-tuning process. 
\end{abstract}

\newpage 

\section{Introduction}

The rise of large language models (LLMs) like GPT has sparked interest across disciplines (including marketing, computer science, economics, psychology and political science) in leveraging these tools to simulate how humans would respond to questions or stimuli in different contexts \citep{aher_using_2023,argyle_out_2023, arora2024express, binz2025foundation,dillion_can_2023,goli2024frontiers, hewitt2024predicting, horton_large_2023,park_generative_2023,qin2024aiturk,toubia2025twin}. If these LLM-simulated experiments accurately approximate real human behavior, the implications for both academics and practitioners are substantial. Academics could use LLMs for pilot experiments to pinpoint stimuli with significant impact, thus improving the efficiency of theory development and experimental design. In fact, several review articles have already been written about this recent yet fast-growing literature \citep{bail2024can,hutson2023guinea,sarstedt2024using,simmons2023large}. Firms could leverage these realistic simulations to explore different ideas and strategies, thereby improving customer insight and product development and optimizing marketing mix variables. Accordingly, in the recent past we have witnessed a large influx of firms offering services leveraging LLMs for customer insights (e.g., \href{https://www.syntheticusers.com/}{Synthetic Users}, \href{https://outset.ai/}{Outset AI}, \href{https://www.nexxt.in/}{Nexxt}, \href{https://www.voxpopme.com/}{Voxpopme}, \href{https://www.evidenza.ai/}{Evidenza}, \href{https://www.expectedparrot.com/}{Expected Parrot}, \href{https://meaningful.app/}{Meaningful}, \href{https://xpolls.ai/}{xPolls}, \href{https://www.ipsos.com/en-uk/large-language-models-market-research}{Ipsos}, \href{https://www.pyxis.ai/}{Pyxis} and \href{https://www.generationlab.org/}{Generation Lab}).

In this paper we focus specifically on leveraging LLMs to simulate counterfactual human behavior. Indeed, many practical applications of LLM simulations involve exploring counterfactual scenarios where one factor is altered while others remain constant. For example, what would happen to sales if prices were increased, how would consumers respond if a message were framed differently, or how would market dynamics shift with the introduction of a new product? 

Early evidence on LLMs' ability to simulate counterfactual human behavior is encouraging but mixed. For example, across 70 nationally-representative U.S. studies, \cite{hewitt2024predicting} find that the correlation between predicted and observed treatment effect size is large ($r=0.85$), but the raw predictions derived from GPT-4 systematically overestimated actual effects, leading to a large (relative to human forecasters) RMSE between predicted and observed effects. Across 11 behavioral economics experiments, \cite{toubia2025twin} find that LLM simulations achieve high predictive accuracy at the individual level, but replicated only about half of the treatment effects observed at the aggregate level. 

The ability of LLMs to simulate counterfactuals will likely improve as more powerful models are developed and the size and quality of the training and validation \emph{data} continues to increase \citep{ludwig2024large, wang2024large}. In this paper, we focus on another key factor, which may be in fact much easier and cheaper to address and which is under the control of any user: the \emph{design} of LLM experiments. We identify a fundamental issue that arises when simply applying best practices from human experiments to LLM experiments. This issue is not specific to a particular model or training dataset and, if left unaddressed, it is likely to cause some amount of residual bias in LLM simulations even as accuracy improves with better models and data. At the same time, we find this issue can be alleviated by carefully adjusting experimental best practices to the world of LLMs.  

In particular, many LLM simulations to date have relied on blinded experimental designs, where simulated participants are not aware of the experimental conditions or hypotheses being tested \citep{argyle_out_2023,brand_using_2023,hewitt2024predicting,horton_large_2023, toubia2025twin}. Blinding\footnote{Also sometimes referred to as ``masking'' \citep{american2019publication}.} has many benefits in human subject research: it reduces demand effects where participants modify their behavior based on what they believe researchers want to see, and it helps maintain ecological validity by presenting scenarios that more closely mirror real-world decision contexts. Following this convention, researchers typically construct LLM prompts that present a single scenario without revealing its experimental nature, similar to how human participants would experience the treatment condition in a between-subjects design. 

Although blinding is considered a best practice in human-subject experiments, we show that it can have unintended consequences in LLM simulations, which may no longer justify its benefits. Unlike human subjects and contexts that exist independently of the experiment, LLM-simulated subjects and contexts are "created on the fly" based on prompts. 
When researchers make LLM-simulated ``subjects'' blind to experimental conditions, variations in the treatment can systematically affect unspecified variables that should remain constant, violating the unconfoundedness assumption typically required for valid causal inference. 

We use demand estimation as an example to empirically illustrate this challenge. We conduct experiments simulating consumer demand for 40 different products at different price points. We find that altering the product's price not only affects the purchase outcomes but also systematically influences other variables that should not be influenced by the treatment—such as pre-treatment variables (e.g., demographic characteristics and customer purchase history) and contextual factors (e.g., price of competing products). This unintended variation in unspecified covariates violates the unconfoundedness assumptions, and in our case it also generates implausible demand curves when compared against an actual experiment with a representative sample of human participants. 

A tempting solution is specifying more variables in the prompts to address this confounding issue. However, we find this approach may sacrifice ecological validity because including more covariates can make the scenarios presented to the LLM less reflective of real-world decision-making contexts. For instance, when competing product information is explicitly mentioned in prompts, these variables become artificially salient in the LLM's decision process, unlike typical consumers who may not always consider such information during routine purchases. This focalism issue creates a fundamental tension between controlling for confounding factors and maintaining ecological validity - a tradeoff that typically does not exist in human-subject experiments where covariates can be measured and controlled for during \textit{data analysis} without making them artificially salient to participants during \textit{data generation}.

We present a theoretical framework illustrating that confoundedness in blind LLM simulations stems from prompts being \emph{ambiguous}. To overcome this fundamental limitation and improve the simulation results, researchers should consider relaxing the blinding requirement and informing the LLM of the experimental design. This approach is not subject to the same confounding issues and holds promise for better generalizable performance as LLMs continue to advance. 

More broadly, we aim to contribute to the growing literature on LLM simulations by helping establish best practices in their design. 
Although extensive literature on experimental design with human subjects has established best practices and universal standards—such as random assignment and blinding—it remains unclear what constitutes best practices in designing LLM simulations and how these might differ from human-subject experimental design principles. Establishing LLM experimental design principles is crucial for enhancing the reliability and validity of LLM simulation results, thereby making them more valuable for both researchers and practitioners. Without clear guidelines on which experimental designs to adopt or avoid, the vast flexibility and sensitivity in LLM simulations can lead to widely varying outcomes \citep{carlson2024use,gao2024take,ludwig2024large, mohammadi2024wait}.  When different simulation approaches yield conflicting results, researchers lack a basis for determining how much weight to place on each finding. By developing standardized design principles, researchers can assign greater weight to well-designed simulations, enhancing the reliability and usefulness of the results. Developing such design principle is essential for advancing the field and ensuring that LLM simulations effectively contribute to our understanding of human behavior.
In particular, by highlighting the trade-offs between blinding, unconfoundedness, and ecological validity, we offer insights that can help researchers and practitioners design more reliable and valid LLM-based experiments. In doing so, our work underscores the importance of adjusting human-subject experimental design principles to the unique characteristics of LLMs, ensuring that these powerful tools can be effectively harnessed to simulate and understand human behavior.

The rest of the paper is organized as follows. Section \ref{sec:motivating_example}  empirically demonstrates the unintended consequences of blinding using demand estimation as an example. Section \ref{sec:control} explores the standard solution of controlling for covariates and highlights its unique challenges in the context of LLM simulations. Section \ref{sec:theory} argues theoretically that this fundamental issue will persist for future models as long as ambiguous prompting strategies are used. Section \ref{sec:unblinding} discusses promising directions for making prompts unambiguous through unblinding, and Section \ref{sec:conclusion} concludes. 

\section{Unintended Confounding in Blind LLM Simulations}\label{sec:motivating_example}

The principle of blinding is a cornerstone of experimental design that consists in withholding information from the participants, such as hiding the design of the experiment or the hypotheses being tested. It is widely used in between-subject experiments with human participants, and is also sometimes implemented in within-subject experiments.\footnote{For example, in a blind taste test comparing Coca-Cola and Pepsi, participants evaluate both beverages (within-subject design) but remain blind to which drink is which, the hypothesis about brand preferences, the tasting order, and the total number of samples they'll evaluate.} 
Blinding helps reduce demand effects and maintains ecological validity by presenting scenarios that mirror real-world decision contexts. This section uses GPT4, demand estimation and one common LLM simulation design to illustrate how blinding creates unique challenges in LLM simulations that are absent in human experiments, and Section \ref{sec:theory} generalizes the findings to other versions of LLMs, other applications, and other simulation designs.

One common way to implement blinding in LLM simulations is to conduct \textit{between-prompt experiments}, where each prompt represents a single experimental condition without revealing the broader experimental context \citep{argyle_out_2023,brand_using_2023,hewitt2024predicting,horton_large_2023,toubia2025twin}. Blinded between-prompt LLM experiments have been used to mimic both between-subject and within-subject experiments with human participants. For example, to mimic a between-subject experiment that studies price sensitivity, researchers might present different prices to the LLM in separate prompts, with each prompt describing only the current price scenario without mentioning the experimental nature of the price variation. To mimic a within-subject conjoint analysis study, researchers might present different product profiles or choice sets to the LLM in separate prompts, with each prompt describing only the current product profile or choice set.

While random assignment in human-subject experiments ensures pre-treatment variables are balanced across conditions and contextual variables remain constant (even when unobserved), LLM-simulated participants and environments are generated dynamically based on prompts. When researchers leave details unspecified in blind LLM prompts, the model may make (implicit or explicit) assumptions about these unspecified variables, which can systematically vary with the treatment and introduce confounding factors. 

We illustrate these issues using demand estimation, a classical problem in economics and marketing, where researchers are often interested in understanding how changes in a product's price affect demand \textit{while holding everything else constant}. We select the top 40 Consumer Packaged Goods (CPG) categories according to \cite{dellavigna_uniform_2019}, and for each category we pick one of the top selling products according to Walmart.com. We record the regular price of each product on Walmart.com around April, 2024. Table \ref{tab:category_product} in the Web Appendix documents this set of 40 products. For each product, we vary the price from 0 to 200\% of the regular price, in 20\% increments. This gives us 11 price points to test for each product ($\{0,20\%,..100\%,..200\%\}$ of the regular price). Throughout the paper, for each prompt template and for each \{category,product\} pair described in Table \ref{tab:category_product}, we vary the price systematically across these 11 price points. We refer to ``relative price'' as the price relative to the regular price ($\{-1,-0.8,..., 0,..., 0.8, 1\}$ where $0$ is the regular price). We use the Open AI API and we ``pin" the model to gpt-4o-mini-2024-07-18 (the latest available at the time of running the experiment) to make sure that the exact same model is used for all our simulations. We adjust the ``System'' part of the prompt to increase the probability of receiving a usable answer from the model (the ``System" portion of the prompt, which can be manipulated on the API, provides meta-instructions to GPT). Unless noted otherwise, for each prompt that corresponds to a \{category, product, price\} combination, we set the temperature to 1 and generate 50 draws to calculate the average response.  

First, we elicit unspecified variables from the LLM. For each \{category, product, price\} combination, we submit a prompt that asks GPT to fill in the blank contextual or pre-treatment variables given the randomized price \citep{kurita2019measuring, nozza2022pipelines, zhang2020hurtful}. Prompt \ref{prompt:past_price} gives an example of such template that asks the past price given the current randomized price. 

\begin{prompt}[Unintended correlation between price and past price]\label{prompt:past_price}
\textbf{System: } You, AI, are a customer. Your task is to fill in the blanks \textunderscore \textunderscore.        
            Return the completed information in comma-separated values, without any extra text.
            
\textbf{User: } Please consider the following product category: \{category\}. 

Suppose you are in a grocery store, and you see the following product in that category: \{product\}. 

The last time you purchased this product, it was priced at \$\underline{\hspace{1cm}} [a number with up to 2 decimal points].

The product is currently priced at: \textcolor{red}{\{price\}}. 

Return example 1: XX.XX
\end{prompt}

Figure \ref{fig:vanilla_correlation_key_variables} demonstrates positive correlations between the price of the focal product and several variables, including the price that the customer paid last time, the price of the competing product, and the expiration days of the product. The x-axis is price relative to the baseline price, allowing us to aggregate across products.\footnote{
Prompts \ref{prompt:competing_price} and \ref{prompt:expiration_days} in Web Appendix \ref{appendix:vanilla_correlation_prompt} document the prompts used to elicit the correlations between randomized price and competing price/expiration date.} Web Appendix \ref{appendix:covariance} presents additional analysis demonstrating the broader scope of these unintended correlations across multiple variables. While these correlations are plausible in observational data, they are undesirable in an experiment where researchers are interested in understanding what happens when changing the price of the focal product while \textit{holding everything else constant.} 

\begin{figure}
    \centering
    \includegraphics[width=1\linewidth]{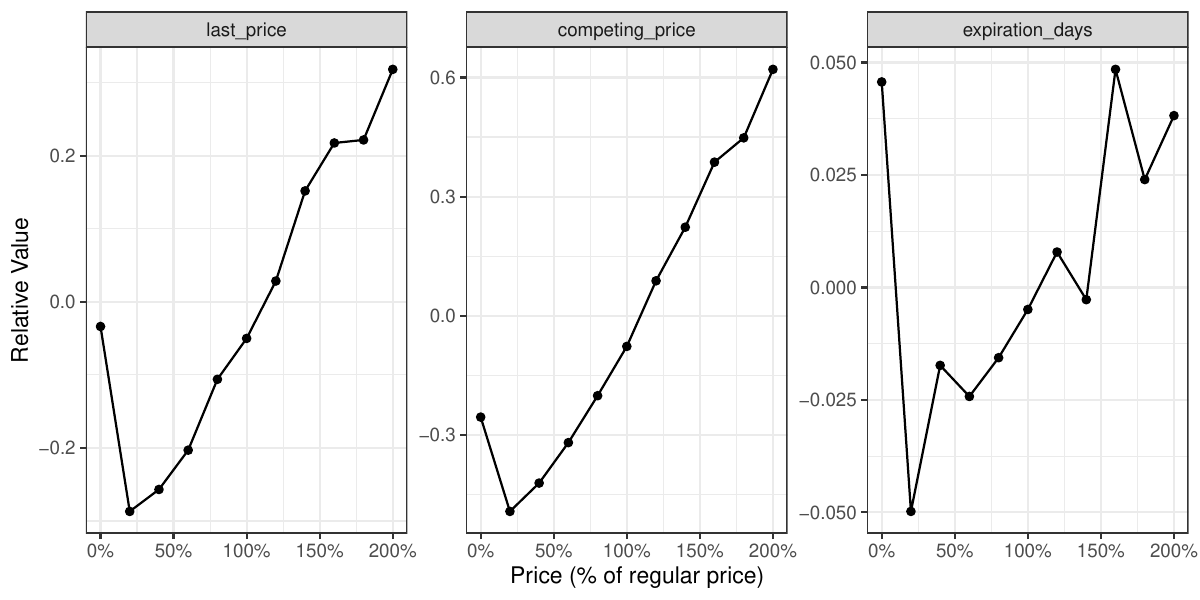}
    \caption{Unintended correlation between price and past price, competing price, and expiration days}
    \label{fig:vanilla_correlation_key_variables}
\end{figure}

Next, we explore whether these unintended correlations are associated with biased estimates of relevant quantities when the goal is simulate counterfactual human behavior. We run an actual survey on human subjects, using a similar prompt in an LLM simulation, and compare the results. The survey with human participants (pre-registered, see https://aspredicted.org/pqys-st6w.pdf) was developed using Qualtrics. Each respondent answered one purchase intention question for each of the 40 products (see Appendix \ref{appendix:category_and_product} for the complete list of categories, products, and prices), with prices randomly drawn for each product and the order of products randomized across respondent.\footnote{We acknowledge that one difference between the study run with humans vs. GPT is that each human answered 40 questions, while on GPT a different independent call to the API was made for each question (following standard practice).} Following our pre-registration, we collected data from 1,000 respondents, and excluded 9 respondents according to pre-registered exclusion criteria (based on attention filters at the beginning of the survey). The respondents were recruited from Prolific, where we paid a premium to obtain a representative sample of the US population (based on census data). We then use Prompt \ref{prompt:vanilla_prompt} to simulate potential purchase decisions at different price points (where again \{category,product\} and \{price\} were varied according to our experimental design, mirroring the experiment with humans), and for each corresponding prompt, we used gpt-4o-mini-2024-07-18, with temperature set to 1, to generate 50 draws to calculate the average purchase probability at each \{category,product,price\} combination. Figure \ref{fig:vanilla_comparison} shows the average demand curves (averaged over the 40 products) obtained from GPT and from human participants. While the curves elicited from human participants follow a classic downward-sloping demand pattern, the GPT-simulated curves follow an inverted-U shape. Although it is understandable that it is challenging for GPT to predict demand when price equals 0 as such scenario is probably not well represented in its training data, the curve remains inverted-U shaped even if we remove the first point corresponding to a price of zero.

\begin{prompt}[Ask Purchase]\label{prompt:vanilla_prompt}
\textbf{System:} You, AI, are a customer. Your task is to fill in the blanks \rule{1cm}{0.15mm}.
Return the completed information in comma-separated values, without any extra text.\\
\textbf{User:} Please consider the following product category: \{category\}.

Suppose you are in a grocery store, and you see the following product in that
category: \{product\}. 

The product is currently priced at \$\{Price\}. Would you or would you not purchase the product? \rule{1cm}{0.15mm} [``purchase'' or ``not purchase'']

Return example:
{purchase}
\end{prompt}

\begin{figure}[htbp!]
    \centering
    \includegraphics[scale = 0.8]{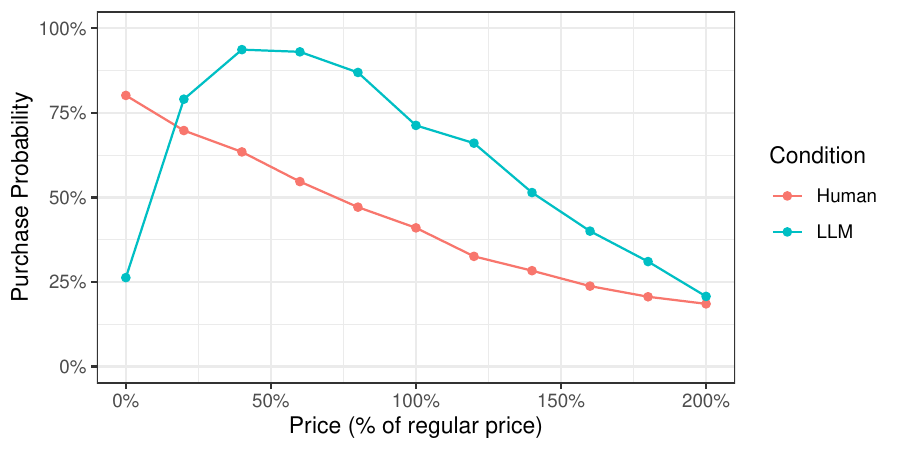}
    \caption{Demand curve elicited from humans vs. LLM, averaged over 40 products}\label{fig:vanilla_comparison}
\end{figure}

These findings are not surprising in hindsight: in real observational data, the price of the focal product should be informative of other factors such as the price of the competing products, especially if we do not know anything about the context, the type of store, and how the price for a product is set. An LLM trained on such data would also naturally make such association. This association can be realistic and useful if researchers intend to use LLMs to elicit the correlation between these other factors and the focal product price. However, this same association is problematic when researchers intend to estimate the causal impact of changing the focal price, while holding everything else constant, because it leads to endogeneity. One common approach to addressing such endogeneity is to control for additional covariates. The next section explores the benefits and novel challenges associated with this approach.

\section{Potential Solution: Controlling for Covariates}\label{sec:control}

When working with real-world data, a common approach for causal identification is to control for covariates. If researchers have controlled for all confounding covariates that affect both the treatment and outcome, the conditional independence assumption, or \textit{ignorability}, is satisfied. If sufficient variation remains after controlling for covariates, then researchers can identify the causal effect of interest. In LLM-simulated experiments, a natural starting point to control for covariates is adding detailed information about individuals and their environment in the prompt. These covariates can range from demographic information, location, preferences, budget, etc.

\cite{argyle_out_2023} introduced an approach for controlling for such covariates. They advocate for running studies on ``silicon samples" of digital personas, which in their case were described by demographic variables and political beliefs that were drawn to align with nationally representative samples. This approach has been picked up by the industry and commercialized by several firms (e.g., \href{https://www.syntheticusers.com/}{Synthetic Users}, \href{https://www.evidenza.ai/}{Evidenza}). Figure \ref{fig:demographic} shows that on average, controlling for covariates, such as demographic variables does indeed improve the results in our case compared to not specifying any detail. However, it does not fully resolve the issue. Web Appendix \ref{appendix:persona} describes the detailed prompts and covariates used, which involve creating of 500 personas for each \{product,category\} pair. 

\begin{figure}[htbp!]
    \centering
    \includegraphics[scale = 0.8]{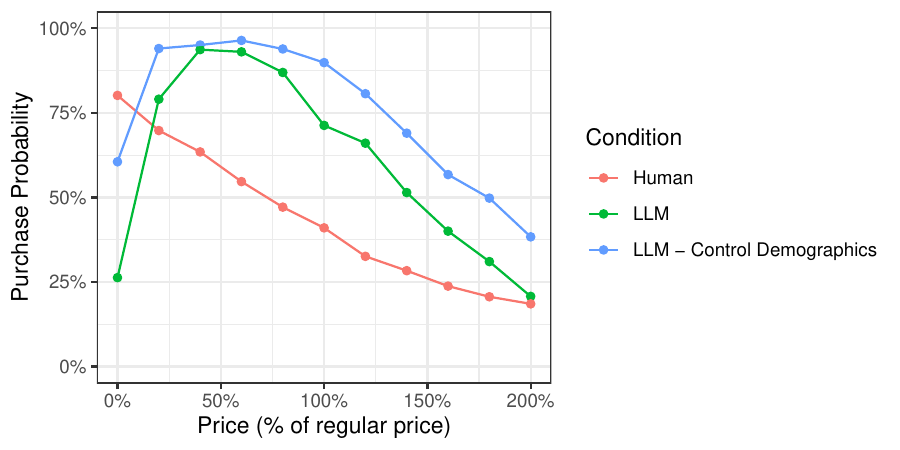}
    \caption{Demand curves elicited by LLM after controlling for demographic variables}\label{fig:demographic}
\end{figure}

A natural question is what constitutes a principled approach to controlling for covariates in LLM simulations and what potential pitfalls exist. When working with human observational data, researchers commonly employ the “kitchen-sink” approach, controlling for many observed pre-treatment variables because a rich set of covariates makes the conditional-independence assumption more credible by capturing confounders. This benefit is strongest with many pre-treatment variables and a large dataset that yields better empirical overlap. While LLM simulations seem well suited here—one can specify many pre-treatment covariates and simulate large samples—an important consideration is the stage at which covariates are controlled for. When working with human data, empirical researchers often control for covariates such as product characteristics and rich user history to address endogeneity during \emph{data analysis}, after the data have already been collected. Hence, analyzing the data does not change the original data-generating process that has already taken place. In contrast, adding covariates via the prompt in an LLM simulation occurs at the \emph{data-generation} stage. Consequently, the additional covariates can change the factors considered by a simulated individual when making decisions.

For example, in the context of demand estimation, one can include contextual information in the prompt such as the price of competing products in the category, which can be used as a proxy for the type of stores that the customer is visiting and the competitive environment that the focal product is facing. However, while adding contextual variables such as competing prices in the LLM simulation may prevent the LLM from using the price treatment to infer these contextual variables, adding them may also lead the LLM-simulated customer to put a heavier emphasis on these factors when making decisions. Figure \ref{fig:fixed_competing_last_price} shows (in the specific example of Coca-Cola) that the demand curve (elicited by simulating the choices of 500 personas) is no longer smooth after fixing such contextual variables,\footnote{We fix the competing price to be the same as the regular price of the focal product.} and the simulated customers' purchase decisions follow a step function: customers will purchase if and only if the price is lower than or equal to the competing price. Web Appendix \ref{appendix:competing_price_focalism} describes the detailed prompts.\footnote{One possible way to fix the non-smoothness of the demand curve is to simulate a distribution of contextual variables, each with difference past and competing prices. However, this is no longer useful for a particular context, e.g., a particular store, because the simulated context deviates from the targeted context. 
In a real-world scenario, we should expect that even in the absence of heterogeneity in context (e.g., in a given store in a given day), there should still exists heterogeneity in consumers, such that the demand curve would be smoother.} Such scenario is implausible in the real world because consumers are often inattentive and also care about other dimensions of the product characteristics. However, in LLM simulation, controlling for competing price to address endogeneity may unintentionally make the controlled covariates artificially salient to the simulated consumer and change the decision environment.

\begin{figure}
    \centering
    \includegraphics[width=0.5\linewidth]{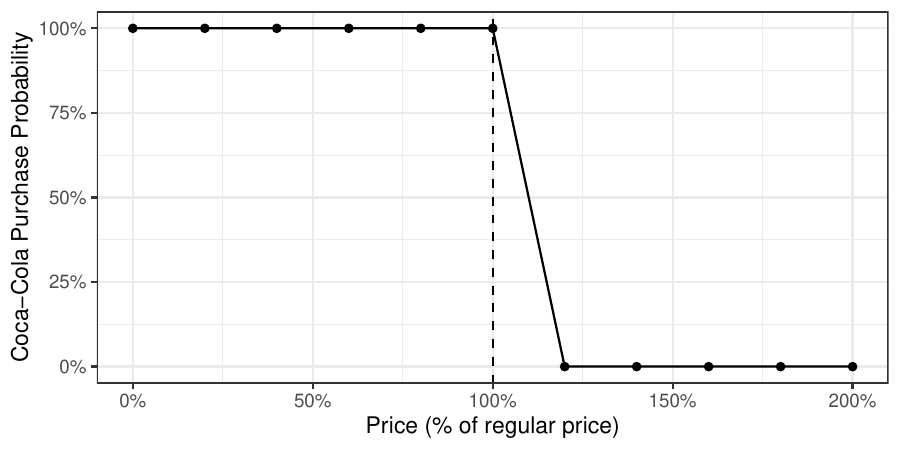}
    \caption{Purchase probability, controlling for demographics and competing price: Coca-Cola example.}
    \label{fig:fixed_competing_last_price}
\end{figure}

Such phenomenon in which decision makers place more emphasis on certain aspects of a choice set or scenario than they would in real-world decisions is known in the consumer behavior literature as focalism \citep{schkade1998does,wilson2000focalism}, and constitutes a threat to ecological validity. For example, in the context of conjoint analysis, \cite{bedi2021damaged} argue that when minor attributes are presented together with major ones, they tend to receive undue attention from respondents (due to focalism), leading to upward biased estimates of the importance of these minor features. While it has been documented that LLMs are prone to focalism \citep{cheng2023compost} and focalism is not an issue that is unique to LLM experiments, the nature of the risk differs as it introduces a new set of trade-off: introducing additional covariates in the simulation may improve unconfoundedness but sacrifice ecological validity. Therefore, unlike standard experimental analysis where controlling covariates typically improve the performance, in LLM simulation, adding more covariates may degrade the performance, because it render the simulated human behavior less realistic. 

To illustrate this trade-off empirically, we leverage a publicly available dataset from \cite{toubia2025twin} collected through multiple survey waves containing diverse measures of consumer cognitive and psychological traits alongside demographic information. Importantly, \cite{toubia2025twin} also replicated our pricing study on their panel, allowing us to relate the results of a study with the exact same design as ours to other variables collected on the same sample of consumers. We implement a sequential covariate addition procedure across
   twelve stages, beginning with a baseline of 14 demographic variables (age, education, income, and others), then systematically adding behavioral
  economics and psychological measures one at a time. This granular approach introduces individual covariates sequentially: tightwad-spendthrift score 
   (stage 2), discount rate and present bias (stage 3), risk and loss aversion (stages 4-5), financial literacy and numeracy (stages 6-7), mental
  accounting (stage 8), decision-making styles (stages 9-10), environmental attitudes (stage 11), and the Big Five personality traits (stage 12), expanding from 14 to 30 total control variables.

Figure \ref{fig:mae_progression} tracks the mean absolute error (MAE) between actual consumer choices in \cite{toubia2025twin}'s replication of our study and LLM predictions as we gradually add covariates. Appendix \ref{appendix:sequential_covariate_addition} provides detailed results showing the non-monotonic effects at each stage, Table \ref{tab:variable_progression} reports the additional covariates added at each stage and Prompt \ref{prompt:full_persona_part1} documents the prompt templates used. In standard experimental analyses with human data, controlling for additional covariates typically improves model performance, provided the sample size is sufficient. However, our results reveal a non-monotonic pattern: initial covariate additions improve accuracy, but beyond a threshold, further additions degrade performance. This deterioration may occurs because explicitly specifying covariates in the prompt induces focalism, in which the LLM assigns disproportionate weight to variables that consumers would typically ignore in naturalistic decision contexts, thereby compromising ecological validity despite reducing confounding.

\begin{figure}[htbp!]
    \centering
    \includegraphics[width=.8\linewidth]{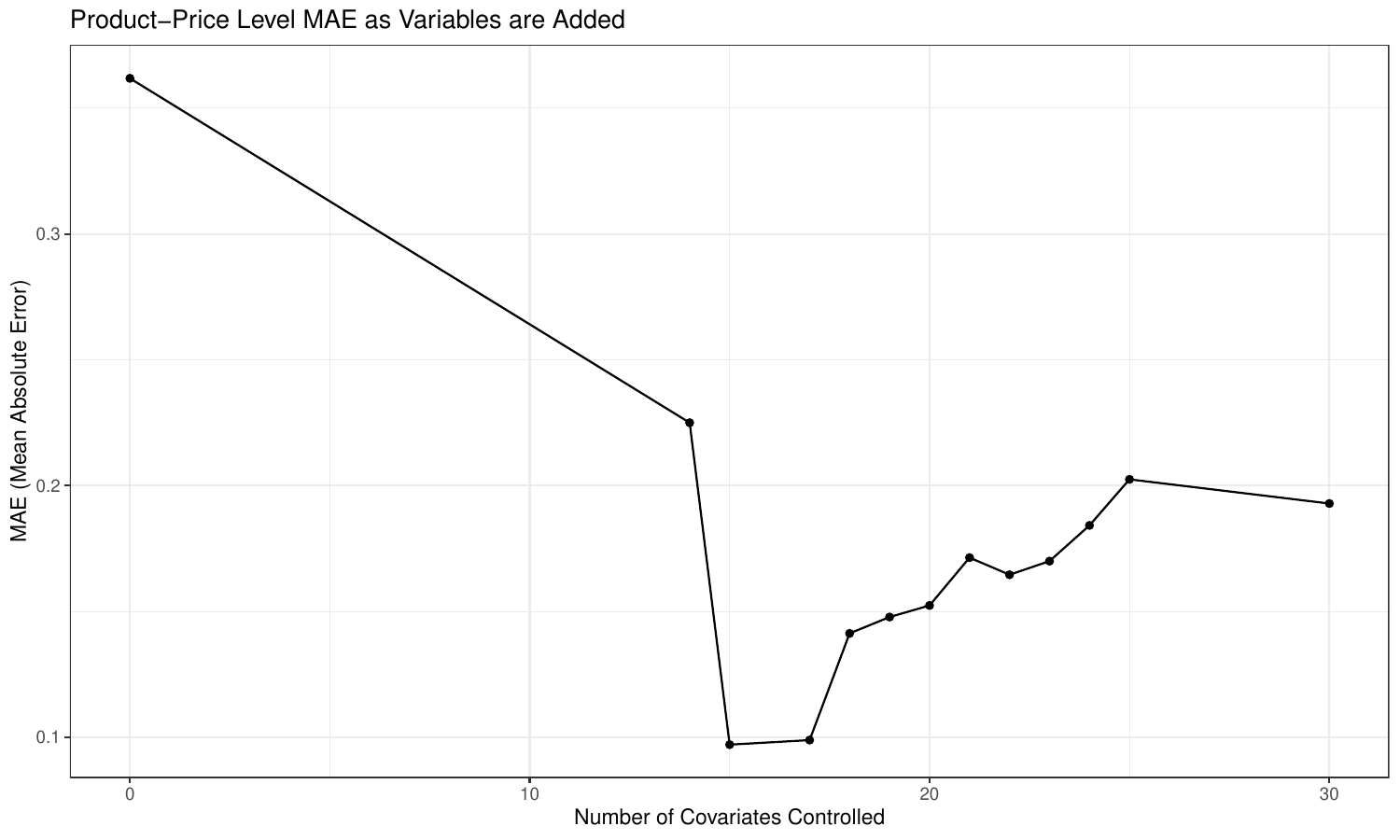}
\caption{Mean absolute error progression as more covariates are controlled in the simulation. The MAE is calculated by comparing the purchase probability for each (product, price) combination.}    \label{fig:mae_progression}
\end{figure}

\newpage 

\section{Core Issue: Ambiguous Prompting Strategy}\label{sec:theory}

So far we have empirically explored one context (demand estimation) with 40 products and one specific version of an LLM. A natural question is whether the issue we have uncovered applies generally to all LLMs and other experiment contexts. This section formalizes the key issue and shows that one key step to harness the full potential of LLM in simulating human behavior lies in avoiding ambiguous prompting strategies. Formally, 

\begin{definition}[Prompting Strategy]
A \emph{prompting strategy} is a function:
\[
s: \mathcal{Q} \rightarrow \mathcal{P},
\]
that maps each question \( q \in \mathcal{Q} \) to a specific prompt \( p \in \mathcal{P} \).
\end{definition}

\begin{definition}[Ambiguous Prompting Strategy]
An \emph{ambiguous prompting strategy} occurs when there exists distinct questions $q_1$ and $q_2$ that have different correct answers, but are mapped to the same prompt $p$: $s(q_1) = s(q_2) = p$ for some prompt \( p \in \mathcal{P} \).
\end{definition}
Such ambiguous prompting strategies should be avoided because - for \textbf{any} given LLM model, no matter how advanced they are, if researchers use an ambiguous prompting strategy, then if it happens to correctly answer one question correctly, $q_1$, it will incorrectly answer its ambiguous counterpart, $q_2$. Web Appendix \ref{appendix:theory} provides a more detailed proof.

An important step for practitioners to avoid ambiguous prompting strategy is to examine whether a specific prompt being used is open to multiple interpretations. For example, consider a simple prompt $p$ under a blinded simulation design:

    \begin{prompt}[A simple ambiguous prompt]\label{prompt:simple}
                Consider the product category: \{category\}.
                
                Suppose you are in a grocery store, and you see the following product in that category: \{product\}. 
                
                The product is priced at: \textcolor{red}{\{price\}}. You decide to \_\_\_ ['purchase' or 'not purchase']
    \end{prompt}

This prompt is \emph{ambiguous} because it allows for at least two interpretations:

\begin{itemize}
    \item $q_1$ (interventional): what is the customer's purchase decision when the price is \textbf{set to} be $\{price\}$?
    \item $q_2$ (observational): what is the customer's purchase decision when the price \textbf{happens to be} $\{price\}$?
\end{itemize}

 These two questions yield different answers. Using the \emph{do-operator} \citep{Pearl2009} to distinguish active intervention from observation, and denoting $Y$ as Purchase and $D$ as Price, the answer to $q_1$ is $a_1: P(Y|do(D))$, where $do(D)$ is randomized and independent of unobserved factors $U$ ($D \perp U$). The answer to $q_2$ is $a_2: P(Y|D)$, where $D$ can be correlated with other confounders $U$ that affect $Y$ ($D \not\perp U$), as in observational data. Since LLMs are trained primarily on observational data, it is plausible that an ambiguous question will also be interpreted in a way that is similar to observational data, hence answering the wrong question. Appendix \ref{appendix:dgp_illustration} provides a DAG illustration of the different types of data generating process that an LLM may mimic when an ambiguous prompt is used. 

As discussed in Section \ref{sec:control}, one might attempt to address this issue by adding more variables to the prompt. However, adding variables to the prompt can actually increase rather than decrease ambiguity from a causal perspective. To illustrate this multiplication of possible interpretations in practice, consider the following detailed example:
\begin{prompt}[More Detailed Scenario with Ambiguous Treatment]\label{prompt:detailed}
\user{You are a customer who has recently read an article about how sugar consumption affects health. You are now in a store and you see Regular Coke (39g sugar per 12 oz, 12 pack cans) priced at \$4.99, with a soda tax of \$0.12/oz. Would you or would you not purchase the product?}
\end{prompt}

This prompt, while detailed, may be even more ambiguous. It permits more valid causal interpretations: 
\begin{enumerate}
    \item \textbf{Price as Treatment:} $P(Y|do(\text{price}), \text{health\_info}, \text{sugar\_content})$ \\
    ``What happens when we experimentally set price to \$4.99, holding fixed that customers are health-aware?"
    
    \item \textbf{Health Information as Treatment:} $P(Y|\text{price}, do(\text{health\_info}), \text{sugar\_content})$:
    ``What happens when we experimentally expose customers to health articles, holding fixed price at \$4.99?"
    
    \item \textbf{Sugar Content as Treatment:} $P(Y|\text{price}, \text{health\_info}, do(\text{sugar\_content}))$:
    ``What happens when we experimentally highlight sugar content, holding fixed price and prior health awareness?"
\end{enumerate}

\section{The value of unblinding}\label{sec:unblinding}

Our theoretical framework points toward a promising direction: designing unambiguous prompts by explicitly communicating the experimental design, which by definition entails unblinding. The causal inference literature \citep{AngristPischke2009, ImbensRubin2015, athey2017econometrics} has long emphasized the importance of clearly specifying estimands and sources of variation. Following this principle, we argue that unblinded prompts should specify both the treatment being randomized and the distribution from which the treatment is drawn. We test unblinding through two approaches: first by examining out-of-the-box LLMs to establish baseline differences, and then by generalizing to fine-tuning experiments.

We begin by examining whether unblinding improves performance in standard, pre-trained language models, establishing a baseline for how prompt design affects predictions without task-specific training. To assess the robustness of the effects, we evaluate a series of GPT models spanning multiple generations, from GPT-3.5-Turbo to the latest models available at the time of writing (GPT-4.1). We operationalize the distinction between blinded and unblinded approaches through two different system prompts.
\begin{prompt}[Blinded System Prompt]\label{prompt:blinded_template}
\textbf{System:} You, AI, are a customer. Your task is to fill in the blank \_\_\_. Return the completed information without extra text.
\end{prompt}

The unblinded condition alters the system prompt to include experimental context:
\begin{prompt}[Unblinded System Prompt]\label{prompt:unblinded_template}
\textbf{System:} You, AI, are an expert in predicting customer behavior. The customer is given a survey on their purchase decision for the \{product\} in \{category\} where the price of the product is randomly and uniformly drawn from \{min\_price\} to \{max\_price\}. The customer is only presented with one price and is blind to this randomization design. The customer is given the following survey. Your task is to fill in the blank \_\_\_. Return the completed information without extra text.
\end{prompt}

This design choice operationalizes our theoretical framework. The unblinded prompt explicitly communicates the data-generating process, including both the randomization mechanism and the support of the price distribution. Figure \ref{fig:model_evolution} reveals a consistent pattern across model generations. Unblinded prompts systematically generate more accurate predictions, with performance improvements ranging from 1\% to over 60\% relative to baseline. The mean absolute error (MAE) reduction persists across all models tested, including reasoning models, and this performance advantage remains substantial as models advance.

\begin{figure}[htbp!]
\centering
\includegraphics[width=\textwidth]{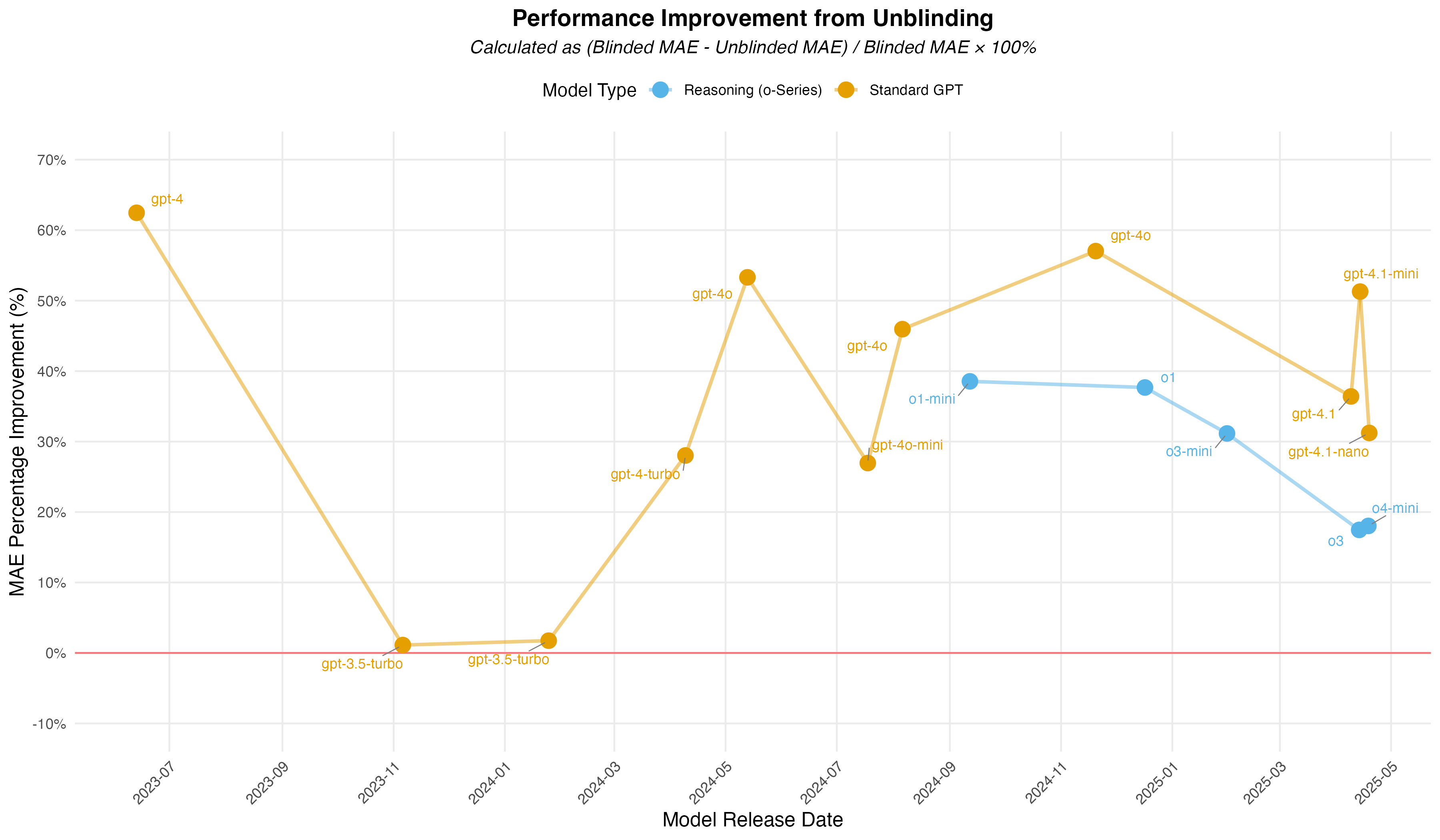}
\caption{Gain from unblinding for various models (ordered by release date and reasoning ability). The figure shows MAE performance improvement when moving from blinded to unblinded prompts, with GPT models ordered by release date. The MAE is calculated by comparing the purchase probability for each (product, price) combination, and then the improvement is normalized by the baseline MAE of the blinded prompts. Positive values indicate superior performance for unblinded prompts.}
\label{fig:model_evolution}
\end{figure}

While these out-of-the-box results demonstrate the value of unblinding, they raise a natural question: will this technique become obsolete as LLMs improve, particularly through fine-tuning? We argue that unblinding is not merely a strategy for evaluating and using pre-trained models but an important, complementary step in the model development and fine-tuning process itself.

To illustrate, we consider fine-tuning a model on actual human data. When performing fine-tuning, researchers must decide what context to provide. For instance, should they reveal the experimental origins of the data (an unblinded approach, as in Prompt \ref{prompt:unblinded_template}) or choose to not disclose such context (a blinded approach, as in Prompt \ref{prompt:blinded_template})? This choice is consequential. Even if fine-tuning helps the model learn the 
correlation between price and purchase, without the context of randomization, it has an insufficient basis for 
interpreting this correlation as causal. Explicitly stating that the price was randomized provides context; hence, it may help the model interpret the learned relationship as causal and thereby improving its potential to generalize. Appendix \ref{sec:fine-tuning} provides a more detailed discussion.

To empirically demonstrate this complementarity, we fine-tune models using our demand estimation survey data. We design an out-of-distribution test by performing a leave-one-category-out cross-validation. We divide the 40 products into four category groups: beverages (including juices, soft drinks, beer, and dairy drinks), refrigerated goods (including meats, dairy products, and frozen foods), snacks and bakery (including chips, cookies, coffee, candy, and ice cream), and household and pet supplies (including cleaning products, paper goods, pet food, and medications). In each fold, we fine-tune a model on three category groups and validate on the held-out fourth category group. This approach tests for true generalization, as the model must learn principles that transfer across fundamentally different product types rather than simply interpolating between similar items. This mirrors real-world applications where practitioners may be interested in forecasting scenarios that historical data has not observed before. 

For each of the four cross-validation folds, we fine-tune GPT-4o-mini models under a 2 (blinded vs. unblinded prompt in fine-tuning stage) × 2 (blinded vs. unblinded prompt in evaluation stage) factorial design, varying the prompt type (blinded Prompt \ref{prompt:blinded_template}  vs. unblinded Prompt \ref{prompt:unblinded_template}) used during both the fine-tuning and evaluation stages. This design allows us to disentangle the effects of unblinding at each stage. To evaluate performance, we test each model on its held-out category, generating 50 responses per product-price pair (with temperature set to 1) and calculating the purchase probability. We then measure the MAE against the true probabilities from our human survey.

Table \ref{tab:mae_survey_only} shows that fine-tuning consistently improves performance over the out-of-the-box model. However, the most accurate results are achieved when the model is both fine-tuned and evaluated with unblinded prompts. This combination reduces the MAE by an additional 15\% compared to the best-performing blinded model, from 0.134 to 0.113. This finding highlights that unblinding offers a significant, nearly costless performance improvement, complementary to the more resource-intensive process of data acquisition and fine-tuning.

\begin{table}[htbp]
\centering
\caption{MAE of Fine-Tuned Models on Survey Data}
\begin{tabular}{lcc}
\toprule
Model & \multicolumn{2}{c}{Evaluation Method} \\
\cmidrule(lr){2-3}
 & Blinded & Unblinded \\
\midrule
Out-of-box GPT-4o-mini & 0.532 & 0.397 \\
Fine-tuned on Survey (Blinded) & 0.134 & 0.128 \\
Fine-tuned on Survey (Unblinded) & 0.130 & \textcolor{red}{\textbf{0.113}} \\
\bottomrule
\end{tabular}

\label{tab:mae_survey_only}
\end{table}

The value of unblinding becomes even more pronounced when fine-tuning models on complex or mixed-quality datasets. While the previous exercise demonstrated unblinding's value with a single-source high-quality dataset, practical applications often involve greater uncertainty. Developers may know a model's general purpose—such as forecasting demand under different pricing strategies—without knowing its \textbf{exact} future applications. Even in a simple setting where a firm only sells two products A and B, the application can include 1) forecasting demand when A and B are priced as usual 2) adjusting the price for product A when B is priced as usual 3) adjusting the price for product B when A is priced as usual, 4) adjusting price for both product A and B, 5) predicting the introduction of a new product C when product A and B are priced as usual. To address this range of potential queries, models may be fine-tuned on diverse training data comprising both experimental datasets, each with specific treatments and covariates, and observational datasets that capture business-as-usual patterns and correlations. While observational data can answer routine questions, it may prove problematic for certain counterfactual analyses. Without additional context during fine-tuning, experimental and observational data risk becoming indistinguishable to the model. Moreover, ambiguous prompts leave the model uncertain about which dataset to leverage for specific questions. Unblinding thus enhances model robustness against irrelevant datasets that may not directly address researchers' intended questions. 

To illustrate the point, we replicate the aforementioned fine-tuning exercise, with the exception that we also added an irrelevant observational dataset that has documented how customers browse online on Amazon \citep{berke2024open}. The dataset is by construction observational and biased because it has only recorded purchases and no recorded non-purchases. Appendix \ref{appendix:amazon_fine_tune} documents the exact prompt we used to incorporate the additional Amazon data in the fine-tuning process. 

Table \ref{tab:mae_survey_amazon} shows that when this irrelevant dataset is included, the performance of models trained with blinded prompts degrades significantly. In contrast, models trained with unblinded prompts prove far more robust. By providing clear context about the data's origin, unblinding enables the model to correctly weight the experimental data and discount the noisy observational data, preserving its predictive accuracy.

\begin{table}[htbp]
\centering
\caption{MAE of Fine-Tuned Models on Mixed (Survey + Amazon) Data}
\begin{tabular}{lcc}
\toprule
Model & \multicolumn{2}{c}{Evaluation Method} \\
\cmidrule(lr){2-3}
 & Blinded & Unblinded \\
\midrule
Out-of-box GPT-4o-mini & 0.532 & 0.397 \\
Fine-tuned on Survey+Amazon (Blinded) & 0.233 & 0.126 \\
Fine-tuned on Survey+Amazon (Unblinded) & 0.145 & \textcolor{red}{\textbf{0.120}} \\
\bottomrule
\end{tabular}

\label{tab:mae_survey_amazon}
\end{table}

\section{Conclusion}\label{sec:conclusion}

This paper identifies a fundamental challenge in using LLMs to simulate experiments with human participants. Unlike human participants who exist independently of experiments, LLM-simulated individuals and environments are generated based on prompts. While LLMs' ability to capture associations from training data is generally beneficial, it creates complications in experimental settings where changing one variable can unintentionally affect other unspecified variables meant to remain constant. This confounding makes it difficult to accurately simulate counterfactual human behavior. 

Our exploration of potential solutions to this challenge reveals important insights for improving LLM simulations. While adding detailed context to control for confounders can help address the confounding issue, our analysis demonstrates that this approach can introduce focalism—causing simulated individuals to overweight explicitly mentioned variables relative to realistic human behavior. This trade-off between reducing confounding and maintaining ecological validity represents a fundamental consideration when designing LLM simulations with detailed prompts.

These empirical observations point to a deeper theoretical issue. Making the LLM blind to the experimental design leads to ambiguous prompts. Because these challenges stem from ambiguous prompting strategies rather than model limitations, our theoretical framework suggests that developing unambiguous prompting strategies constitutes a fundamental principle that complements other model improvement techniques.

Building on this theoretical understanding, our empirical findings demonstrate that unblinded experimental designs offer a particularly promising path forward. By explicitly communicating the randomization mechanism and experimental context, unblinded prompts reduce ambiguity and enable more accurate simulations. This approach shows consistent improvements across model generations and proves especially valuable when combined with fine-tuning, where it enhances the model's ability to distinguish causal from correlational patterns in training data. The complementary nature of unblinding and fine-tuning suggests that developing LLMs capable of accurately simulating counterfactual human behavior when provided with clear, unambiguous prompts represents both a crucial and achievable direction for advancing the field.

Our findings have important practical implications for different stakeholders in the LLM ecosystem. For researchers evaluating LLM performance, our results caution against using blinded designs when assessing simulation capabilities. Poor performance in such settings may reflect prompt ambiguity rather than model limitations, making it impossible to properly diagnose model capabilities. For academic researchers studying LLMs, our work highlights the need to establish consensus on unambiguous prompting strategies. Without such consensus, different research teams may map distinct questions to identical prompts, leading to conflicting interpretations that no model can simultaneously satisfy. Perhaps most importantly, our findings suggest avoiding ambiguous prompts during the development and fine-tuning of LLMs. When prompts permit multiple interpretations, improving performance on one interpretation may degrade performance on others - an unnecessary trade-off that could be avoided through clearer prompting strategies, especially given uncertainty about future applications.

Our paper highlights a broader insight regarding the use of LLMs: rather than simply applying human subject protocols to LLMs, researchers must develop new experimental practices tailored to LLMs' unique properties. The field of LLM simulations remains nascent, and best practices have yet to be established. Our work highlights the need for a new set of principles for designing LLM simulations and evaluation frameworks. We hope this will guide future research in developing more reliable methods for leveraging LLMs in experimental research while remaining mindful of their fundamental differences from human subjects.

\

\textbf{Funding and Competing Interests Declarations}

All authors certify that they have no affiliations with or involvement in any organization or entity with any financial interest or non-financial interest in the subject matter or materials discussed in this manuscript. The authors have no funding to report.

\bibliographystyle{apalike}

\bibliography{ChatGPT, additional_references}
\appendix
\counterwithin{figure}{section}
\counterwithin{table}{section}
\newpage
\section{Category, Product and Price Information}\label{appendix:category_and_product}
\begin{table}[htbp]
\tiny
\setlength{\tabcolsep}{4pt}  
\centering
\begin{threeparttable}

\caption{Categories, products and regular prices}
\label{tab:category_product}
\begin{tabular}{lp{12cm}r}
\toprule
Category & Product & Price (\$) \\
\midrule
Fruit Juice & \href{https://www.walmart.com/ip/110946086}{Capri Sun Variety Pack with Fruit Punch, Strawberry Kiwi \& Pacific Cooler Juice Box Pouches, 30 ct Box, 6 fl oz Pouches} & 9.43 \\
Fruit Drinks & \href{https://www.walmart.com/ip/558416376}{Kool Aid Jammers Variety Pack with Tropical Punch, Grape \& Cherry Kids Drink 0\% Juice Box Pouches, 30 Ct Box, 6 fl oz Pouches} & 7.27 \\
Baby Milk and Milk Flavoring & \href{https://www.walmart.com/ip/819219798}{Horizon Organic Shelf-Stable Whole Milk Boxes, 8 oz., 12 Pack} & 13.98 \\
Soup & \href{https://www.walmart.com/ip/10450904}{Maruchan Ramen Noodle Chicken Flavor Soup, 3 Oz, 12 Count Shelf Stable Package} & 9.97 \\
Cat Food - Wet Type & \href{https://www.walmart.com/ip/10534976}{Purina Fancy Feast Chicken Feast Classic Grain Free Wet Cat Food Pate - 3 oz. Can} & 0.88 \\
Pet Supplies - Dog Food & \href{https://www.walmart.com/ip/1446707820}{Purina Dog Chow Complete, Dry Dog Food for Adult Dogs High Protein, Real Chicken, 44 lb Bag} & 29.17 \\
Snacks - Potato Chips & \href{https://www.walmart.com/ip/677669806}{Lay's Classic Potato Snack Chips, Party Size, 13 oz Bag} & 5.44 \\
Snacks - Tortilla Chips & \href{https://www.walmart.com/ip/433078517}{Doritos Nacho Cheese Tortilla Snack Chips, Party Size, 14.5 oz Bag} & 5.94 \\
Cereal - Ready to Eat& \href{https://www.walmart.com/ip/207302221}{Cinnamon Toast Crunch Breakfast Cereal, Crispy Cinnamon Cereal, Family Size, 18.8 oz} & 4.93 \\
Cookies & \href{https://www.walmart.com/ip/10295206}{Little Debbie Oatmeal Creme Pies, 12 ct, 16.2 oz} & 2.68 \\
Ground and Whole Bean Coffee & \href{https://www.walmart.com/ip/971362035}{Folgers Classic Roast Ground Coffee, Medium Roast, 40.3-Ounce Canister} & 13.24 \\
Soft Drinks - Carbonated & \href{https://www.walmart.com/ip/12166733}{Coca-Cola Soda Pop, 12 fl oz, 12 Pack Cans} & 8.26 \\
Bottled Water & \href{https://www.walmart.com/ip/376302514}{OZARKA Brand 100\% Natural Spring Water, 16.9-ounce plastic bottles (Pack of 35)} & 19.96 \\
Candy - Chocolate & \href{https://www.walmart.com/ip/Hershey-s-Individually-Wrapped-Bars-Milk-Chocolate-6-ct/10452239}{Hershey's Milk Chocolate Candy, Bars 1.55 oz, 6 Count} & 6.48 \\
Candy - Non-Chocolate & \href{https://www.walmart.com/ip/52796439}{HARIBO Goldbears Original Gummy Bears, 28.8oz Stand Up Bag} & 6.48 \\
Soft Drinks - Low Calorie & \href{https://www.walmart.com/ip/10535134}{Coca-Cola Zero Sugar Soda Pop, 16.9 fl oz, 6 Pack Cans} & 5.18 \\
Frozen Italian Entrees & \href{https://www.walmart.com/ip/14940646}{Smart Ones Three Cheese Ziti Marinara Frozen Meal, 9 Oz Box} & 2.26 \\
Frozen Foods & \href{https://www.walmart.com/ip/36618723}{Great Value All Natural Chicken Wing Sections, 4 lb (Frozen)} & 12.98 \\
Ice Cream & \href{https://www.walmart.com/ip/13281759}{Haagen Dazs Coffee Ice Cream, Gluten Free, Kosher, 14.0 oz} & 4.18 \\
Frozen Novelties & \href{https://www.walmart.com/ip/50863392}{Pop-Ice Assorted Fruit Freezer Ice Pops, Gluten-Free Snack, 1.5 oz, 80 Count Fruit Pops} & 6.17 \\
Lunchmeat - Sliced - Refrigerated & \href{https://www.walmart.com/ip/10292642}{Oscar Mayer Chopped Ham \& Water product Deli Lunch Meat, 16 Oz Package} & 4.33 \\
Frankfurters - Refrigerated & \href{https://www.walmart.com/ip/39976745}{Oscar Mayer Classic Uncured Beef Franks Hot Dogs, 10 ct Pack} & 3.94 \\
Refrigerated Bacon & \href{https://www.walmart.com/ip/21268963}{Oscar Mayer Fully Cooked Original Bacon, 2.52 oz Box} & 4.27 \\
Refrigerated Entrees & \href{https://www.walmart.com/ip/38227609}{John Soules Foods Chicken Breast Fajita Strips, Refrigerated, 16oz, 18g Protein per 3oz Serving Size} & 5.98 \\
Dairy Products & \href{https://www.walmart.com/ip/10291058}{Land O Lakes Salted Stick Butter, 16 oz, 4 Sticks} & 5.28 \\
Yogurt - Refrigerated & \href{https://www.walmart.com/ip/21291512}{Chobani Non-Fat Greek Yogurt, Vanilla Blended 32 oz, Plastic} & 5.58 \\
Refrigerated Deli Meats & \href{https://www.walmart.com/ip/49342259}{Goya Cooked Ham 16 oz} & 29.99 \\
Dairy - Milk - Refrigerated & \href{https://www.walmart.com/ip/10450114}{Great Value Milk Whole Vitamin D Gallon} & 3.92 \\
Bakery - Fresh Cakes & \href{https://www.walmart.com/ip/16777549}{Little Debbie Zebra Cakes, 13 oz} & 2.68 \\
Fresh Eggs & \href{https://www.walmart.com/ip/10324126}{Eggland's Best Classic Extra Large White Eggs, 12 count} & 3.18 \\
Fresh Fruit & \href{https://www.walmart.com/ip/44390957}{Fresh Raspberries, 12 oz Container} & 4.74 \\
Beer & \href{https://www.walmart.com/ip/22563370}{Stella Artois Lager, 12 Pack, 11.2 fl oz Glass Bottles, 5\% ABV, Domestic Beer} & 15.73 \\
Light Beer (Low Calorie/Alcohol)& \href{https://www.walmart.com/ip/Bud-Light-Beer-24-Pack-12-fl-oz-Aluminum-Cans-4-2-ABV-Domestic-Lager/10984488}{Bud Light Beer, 24 Pack, 12 fl oz Aluminum Cans, 4.2\% ABV, Domestic Lager} & 20.98 \\
Detergents - Heavy Duty - Liquid & \href{https://www.walmart.com/ip/30260539}{Purex Liquid Laundry Detergent Plus OXI, Stain Defense Technology, 128 Fluid Ounces, 85 Wash Loads} & 9.97 \\
Cleaning Supplies & \href{https://www.walmart.com/ip/10291025}{ARM \& HAMMER Pure Baking Soda, For Baking, Cleaning \& Deodorizing, 1 lb Box} & 1.54 \\
Toilet Tissue & \href{https://www.walmart.com/ip/708542578}{Angel Soft Toilet Paper, 9 Mega Rolls, Soft and Strong Toilet Tissue} & 6.68 \\
Paper Towels & \href{https://www.walmart.com/ip/2101922242}{Bounty Select-a-Size Paper Towels, 12 Double Rolls, White} & 22.18 \\
Batteries & \href{https://www.walmart.com/ip/16663048}{Duracell Coppertop AA Battery, Long Lasting Double A Batteries, 16 Pack} & 15.97 \\
Pain Remedies - Headache & \href{https://www.walmart.com/ip/Tylenol-Extra-Strength-Caplets-100/893337}{Tylenol Extra Strength Caplets with 500 mg Acetaminophen, 100 Ct} & 10.97 \\
Cold Remedies -Adult & \href{https://www.walmart.com/ip/Equate-Value-Size-Honey-Lemon-Cough-Drops-with-Menthol-160-Count/54320956}{Equate Value Size Honey Lemon Cough Drops with Menthol, 160 Count} & 4.68 \\
\bottomrule
\end{tabular}
\begin{minipage}{\linewidth}
\footnotesize
\vspace{0.2cm}
\textit{Note:} The 40 categories were taken from \cite{dellavigna_uniform_2019} and the top products and their regular prices from Walmart.com as of April, 2024. Due to the dynamic nature of retail pricing, actual prices may vary.
\end{minipage}
\end{threeparttable}

\end{table}

\pagebreak

\section{Simple prompts for eliciting correlation}\label{appendix:vanilla_correlation_prompt}

Prompts \ref{prompt:competing_price} and \ref{prompt:expiration_days} document the prompts used to elicit the correlations between randomized price and competing price/expiration days shown in Figure \ref{fig:vanilla_correlation_key_variables}. 

\begin{prompt}[Correlation between price and competing price]\label{prompt:competing_price}
\textbf{System: } You, AI, are a customer. Your task is to fill in the blanks \textunderscore \textunderscore.        
                Return the completed information in comma-separated values, without any extra text.
                
\textbf{User: } Please consider the following product category: \{category\}. 

Suppose you are in a grocery store, and you see the following product in that category: \{product\}. 

The product is currently priced at: \textcolor{red}{\{price\}}. The price of a similar competing product from a different brand is \underline{\hspace{1cm}} [a number with up to 2 decimal points].

Return example 1: XX.XX
\end{prompt}

\begin{prompt}[Correlation between price and expiration days]\label{prompt:expiration_days}
\textbf{System: } You, AI, are a customer. Your task is to fill in the blanks \textunderscore \textunderscore.        
                Return the completed information in comma-separated values, without any extra text.
                
\textbf{User: } Please consider the following product category: \{category\}. 

Suppose you are in a grocery store, and you see the following product in that category: \{product\}. 

The product is currently priced at: \textcolor{red}{\{price\}}. The expiration date of the product is \underline{\hspace{1cm}} [a whole number] days from now.

Return example 1: 10
\end{prompt}

\newpage
\section{Covariance of unspecified variables}\label{appendix:covariance}

This section demonstrates how prevalent variables covary with each other, for which we use a more detailed prompt \ref{prompt:demonstrate_correlation} that elicits demographic, contextual, and pre-treatment variables. Figure \ref{fig:covariance_matrix} summarizes the covariance matrix between these variables. On one hand, this covariance matrix demonstrates that the LLM is capable of simulating many realistic correlations between covariates. For example, household income is positively correlated with budget. However, several covariates are correlated with the price change variable. Specifically, the focal price is correlated with product characteristics such as expiration days of the product, contextual variables such as competing prices, and pre-treatment variables such as purchase history. In an experiment with humans, these variables would either be held constant or be randomly distributed across treatment conditions in a way that is orthogonal to the treatment. 

\begin{figure}[htbp!]
    \centering
    \includegraphics[scale = 0.8]{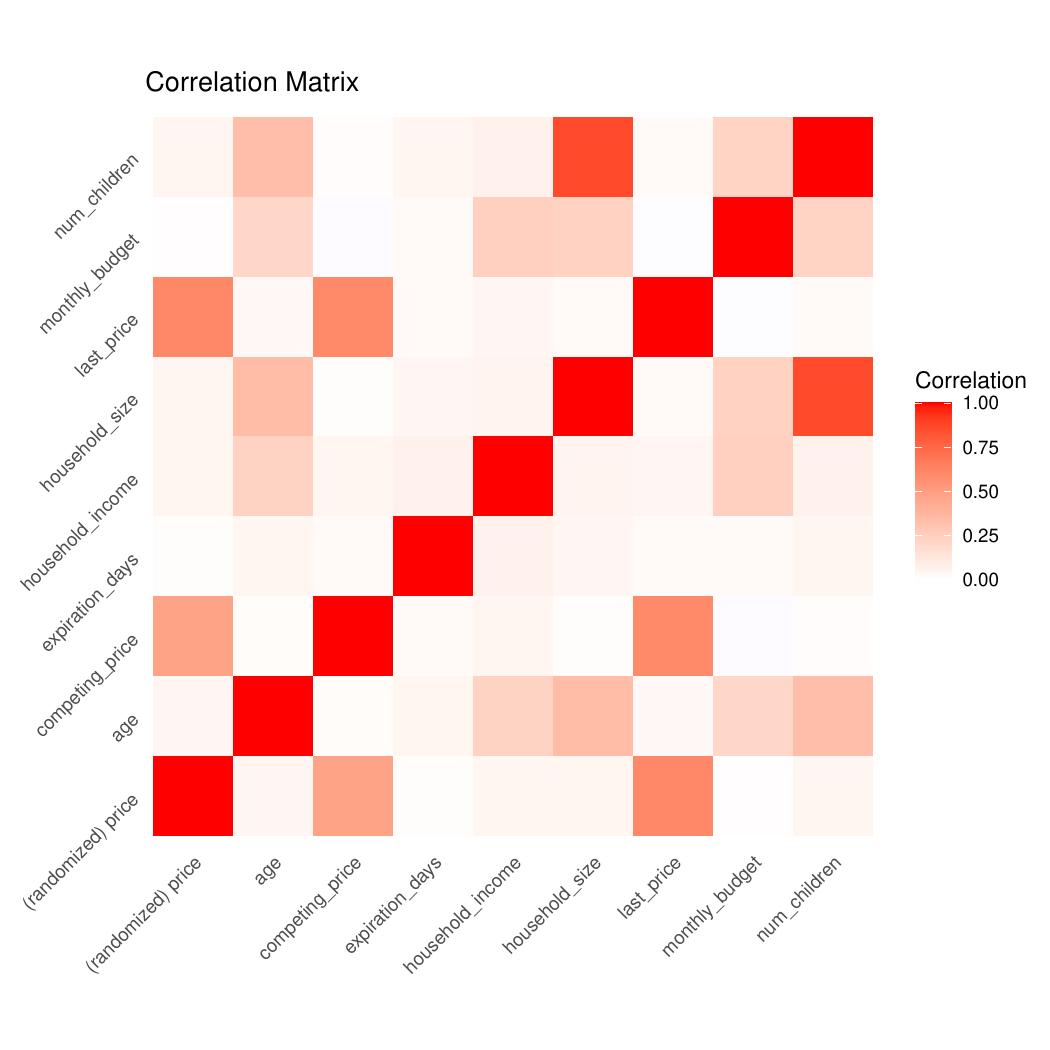}
    \caption{Covariance matrix of demographic, contextual, and pre-treatment variables}\label{fig:covariance_matrix}
\end{figure}

\begin{figure}[htbp!]
\begin{prompt}[Correlation between multiple variables]\label{prompt:demonstrate_correlation}
\textbf{System:} You, AI, are a customer. Your task is to fill in the blanks \rule{1cm}{0.15mm}.
Return the completed information in comma-separated values, without any extra text.\\
\textbf{User:} You are a consumer with the following characteristics:\\
Age: \rule{1cm}{0.15mm} [a whole number]\\
Gender: \rule{1cm}{0.15mm}\\
Education level: \rule{1cm}{0.15mm}\\
Household income: \rule{1cm}{0.15mm} [a whole number]\\
Occupation: \rule{1cm}{0.15mm}\\
Ethnicity: \rule{1cm}{0.15mm}\\
Marital status: \rule{1cm}{0.15mm}\\
Household size: \rule{1cm}{0.15mm} [a whole number]\\
Number of children: \rule{1cm}{0.15mm} [a whole number]\\
State of residence: \rule{1cm}{0.15mm} [state]\\
Home ownership: \rule{1cm}{0.15mm} [e.g., ``own,'' ``rent'']

Please consider the following product category: \textbf{\{category\}}.\

Suppose you are in a grocery store, and you see the following product in that category: \textbf{\{product\}}.

You have purchased this product \rule{1cm}{0.15mm} [e.g., ``frequently,'' ``occasionally,'' ``rarely''] in the past. The last time you saw this product, it was priced at \$\rule{1cm}{0.15mm} [a number with up to 2 decimal points]. You have \rule{1cm}{0.15mm} [e.g., ``a lot of,'' ``some,'' ``limited''] storage space at home. Your monthly grocery budget is \rule{1cm}{0.15mm} [a whole number].

The expiration date of the product is \rule{1cm}{0.15mm} [a whole number] days from now. The product is currently priced at \$\textbf{\{Price\}}. A similar competing product from a different brand is priced at \$\rule{1cm}{0.15mm} [a number with up to 2 decimal points]. 

Would you or would you not purchase the product? \rule{1cm}{0.15mm} [``purchase'' or ``not purchase'']

Return example: 35, female, bachelor's degree, 50000, software engineer, Asian, married, 3, 1, \\
\text{CA, own, occasionally, 2.99, some, 800, 60, 3.49, purchase}
\end{prompt}
\caption{Eliciting correlation between relevant variables}
\end{figure}

\section{Simulation that controls for demographic variables}\label{appendix:persona}

We use Prompt \ref{prompt:control_demographic} to first generate the distribution of demographic variables and then control for demographic variables. First, for each \{category,product\}, we leave the demographic variables as blank, and use the pinned version of gpt-4o-mini-2024-07-18 to generate 500 replications (personas) with a temperature of 1, giving us the joint distribution of demographic variables for that product. See Prompt \ref{prompt:control_demographic}. To avoid imposing any prior assumption on Price, we leave \{Price\} as " \rule{1cm}{0.15mm} [a number with up to 2 decimal points]" in this step. Then in the second step, we use the same prompt template but with the demographic variables filled in, i.e., all variables except the purchase decision filled in and the price being randomized. For each \{category,product\} and for each persona for that product, we vary the price, and elicit the purchase decision, while setting the temperature to 0, and then aggregate across the 500 personas to calculate the purchase probability. Then we aggregate across all products at the same relative price to construct the aggregate demand curve. 

\begin{prompt}[Control for demographic variables]\label{prompt:control_demographic}
\textbf{System:} You, AI, are a customer. Your task is to fill in the blanks \rule{1cm}{0.15mm}.
Return the completed information in comma-separated values, without any extra text.\\
\textbf{User:} You are a consumer with the following characteristics:\\
Age: \rule{1cm}{0.15mm} [a whole number]\\
Gender: \rule{1cm}{0.15mm}\\
Education level: \rule{1cm}{0.15mm}\\
Household income: \rule{1cm}{0.15mm} [a whole number]\\
Occupation: \rule{1cm}{0.15mm}\\
Ethnicity: \rule{1cm}{0.15mm}\\
Marital status: \rule{1cm}{0.15mm}\\
Household size: \rule{1cm}{0.15mm} [a whole number]\\
Number of children: \rule{1cm}{0.15mm} [a whole number]\\
State of residence: \rule{1cm}{0.15mm} [state]\\
Home ownership: \rule{1cm}{0.15mm} [e.g., ``own,'' ``rent'']

Please consider the following product category: \textbf{\{category\}}.

Suppose you are in a grocery store, and you see the following product in that category: \textbf{\{product\}}.

The product is currently priced at \rule{1cm}{0.15mm} [a number with up to 2 decimal points].

Would you or would you not purchase \textbf{\{product\}}? \rule{1cm}{0.15mm} [``purchase'' or ``not purchase'']\\
Return example: \text{35, female, bachelor's degree, 50000, software engineer, Asian,married, 3, 1, }\\
\text{CA, own, 3.99, purchase}
\end{prompt}

\subsection{Simulation that also controls for competing price}\label{appendix:competing_price_focalism}

To illustrate the unintended effect of adding covariates in simulation, we focus on a product, Coca-cola, and asks what happens to the purchase probability when the price of a competing product is set at a certain level. We follow the exact same two-step procedure as Web Appendix \ref{appendix:persona}, except that we have fixed the price of a competing product to be the same as the regular price of the focal product when eliciting purchase decisions from the 500 personas. 

\begin{prompt}[Control for demographic variables+ competing price]\label{prompt:control_demographic_and_competing_price}
\textbf{System:} You, AI, are a customer. Your task is to fill in the blanks \rule{1cm}{0.15mm}.
Return the completed information in comma-separated values, without any extra text.\\
\textbf{User:} You are a consumer with the following characteristics:\\
Age: \{age\}\\
Gender: \{gender\}\\
Education level: \{education\}\\
Household income: \{income\}\\
Occupation: \{occupation\}\\
Ethnicity: \{ethnicity\}\\
Marital status: \{marital status\}\\
Household size: \{household size\}\\
Number of children: \{number of children\}\\
State of residence: \{state\}\\
Home ownership: \{home ownership\}

Please consider the following product category: Soft Drinks - Carbonated.

Suppose you are in a grocery store, and you see the following product in that category: Coca-Cola Soda Pop, 12 fl oz, 12 Pack Cans.

The product is currently priced at \$\textbf{\{Price\}}. 
A similar competing product from a different brand is priced at \$8.26.

Would you or would you not purchase Coca-Cola Soda Pop, 12 fl oz, 12 Pack Cans? \rule{1cm}{0.15mm} [``purchase'' or ``not purchase'']\\

Return example: purchase

\end{prompt}

This mimics the scenario of an actual store, where the price of a competing product is set uniformly at a certain level. Different customers, each with their own characteristics, will come to the store and make their purchase decision. While it is plausible that some customers may notice the price of the competing product and use it as a reference point to make purchase decisions, our step-wise non-smooth demand curve is implausible in real-world scenarios where not all customers are attentive or are aware of the price of all other competing products. 

\section{Non-monotonic impact of controlling for covariates}\label{appendix:sequential_covariate_addition}

Table \ref{tab:variable_progression} gives an overview of the variables used and the corresponding results documented in Section \ref{sec:control}. 
 
\begin{table}[htbp!]
    \centering
    \caption{Sequential Covariate Addition in Sensitivity Analysis}
    \label{tab:variable_progression}
    \begin{tabular}{cccl c}
\toprule
Chunks & Total & New & New Covariates & MAE \\
 & Covariates & Covariates & (Names) & \\
\midrule
1 & 14 & 14 & Age, US Citizenship, Education Level, Employmen... & 0.2250 \\
2 & 15 & 1 & Tightwad-Spendthrift Score & 0.0971 \\
3 & 17 & 2 & Discount Rate, Present Bias & 0.0989 \\
4 & 18 & 1 & Risk Aversion & 0.1413 \\
5 & 19 & 1 & Loss Aversion & 0.1478 \\
6 & 20 & 1 & Financial Literacy & 0.1524 \\
7 & 21 & 1 & Numeracy & 0.1714 \\
8 & 22 & 1 & score mentalaccounting & 0.1646 \\
9 & 23 & 1 & score maximization & 0.1700 \\
10 & 24 & 1 & score minimalism & 0.1842 \\
11 & 25 & 1 & score GREEN & 0.2025 \\
12 & 30 & 5 & score agreeableness, score extraversion, score ... & 0.1929 \\
\bottomrule
\end{tabular}

    \begin{tablenotes}
    \small
    \item \textit{Note:} This table presents the sequential addition of covariates in the sensitivity analysis. Stage 1 includes sociodemographic
    characteristics (age, citizen status, education, employment status, household size, marriage, political affiliation, political views, race, region,
    religion, religious attendance, sex, and total family income). Stages 2-11 add individual behavioral economics and psychological measures. Stage 12 adds
     Big Five personality traits (agreeableness, extraversion, neuroticism, openness, conscientiousness). Mean Absolute Error (MAE) represents prediction
    error at each stage.
    \end{tablenotes}
    \end{table}

Prompt \ref{prompt:full_persona_part1} shows the final template that includes all the covariates we tested at Stage 12. At stage 1, only the demographic variables (the first 14 items) were included. Each subsequent stage systematically added behavioral and psychological measures: tightwad-spendthrift (stage 2), temporal discounting (stage 3), risk aversion (stage 4), loss aversion (stage 5), financial literacy (stage 6), numeracy (stage 7), mental accounting (stage 8), maximization (stage 9), minimalism (stage 10), environmental attitudes (stage 11), and finally the Big Five personality traits (stage 12).

\begin{prompt}[Prompt Templates Based on Detailed Covariates: Part I]\label{prompt:full_persona_part1}
\scriptsize
\textbf{System: } You, AI, are a customer. Your task is to fill in the blanks \rule{1cm}{0.15mm}.
Return the completed information in comma-separated values, without any extra text.\\

\textbf{User:} You are a consumer with the following characteristics:

\# Demographics:\\
- Geographic region: \{region\}\\
- Gender: \{sex\}\\
- Age: \{age\}\\
- Education level: \{education\}\\
- Race: \{race\}\\
- Citizen of the US: \{citizen\_status\}\\
- Marital status: \{marriage\}\\
- Religion: \{religion\}\\
- Religious attendance: \{religious\_attendance\}\\
- Political affiliation: \{political\_affiliation\}\\
- Income: \{total\_family\_income\}\\
- Political views: \{political\_views\}\\
- Household size: \{household\_size\}\\
- Employment status: \{employment\_status\}

\# Tightwad-Spendthrift: \{score\_ST-TW\} (\{pct\_spendthrift\} percentile)\\
\textit{$<$note: The score ranges from 4 to 26. Lower scores (4-11) indicate difficulty spending money, while higher scores (19-26) indicate difficulty controlling spending.$>$}

\# Discount, Present Bias:\\
\textit{$<$note: These are implied rates computed from your time-value of money preferences. Higher values of the discount rate imply greater impatience. Higher values of present bias imply greater departure from normative economic behavior.$>$}\\
- Discount: \{score\_discount\} (\{pct\_discount\} percentile)\\
- Present Bias: \{score\_presentbias\} (\{pct\_presentbias\} percentile)

\# Risk aversion: \{score\_riskaversion\} (\{pct\_riskaversion\} percentile)\\
\textit{$<$note: Higher scores indicate a greater tendency for risk aversion in a choice between a sure-amount and lottery payout.$>$}
\end{prompt}

\addtocounter{promptcounter}{-1}
\begin{prompt}[Prompt Templates Based on Detailed Covariates: Part II]\label{prompt:full_persona_part2}
\scriptsize
\# Loss aversion: \{score\_lossaversion\} (\{pct\_lossaversion\} percentile)\\
\textit{$<$note: Higher scores indicate a greater tendency for loss aversion in a choice between a sure-amount and a lottery payout.$>$}

\# Financial Literacy: \{score\_finliteracy\} (\{pct\_finliteracy\} percentile)\\
\textit{$<$note: The score ranges from 0 to 8, and a higher score indicates you correctly answered more questions related to general financial literacy.$>$}

\# Numeracy: \{score\_numeracy\} (\{pct\_numeracy\} percentile)\\
\textit{$<$note: The score ranges from 0 to 8, and a higher score indicates you correctly answered more questions related to numeracy.$>$}

\# Mental Accounting: \{score\_mentalaccounting\} (\{pct\_mentalaccounting\} percentile)\\
\textit{$<$note: The score ranges from 0 to 100 percent, and higher scores indicate a greater adherence to the principles of mental accounting proposed by Thaler: segregate gains, integrate losses, segregate a small gain from a large loss, and integrate a small loss with a large gain.$>$}

\# Maximization: \{score\_maximization\} (\{pct\_maximization\} percentile)\\
\textit{$<$note: The score ranges from 1 to 5, and higher scores indicate a tendency to optimize rather than satisfice when making decisions.$>$}

\# Minimalism: \{score\_minimalism\} (\{pct\_minimalism\} percentile)\\
\textit{$<$note: The score ranges from 1 to 5, and a higher score indicates a higher preference for minimalism.$>$}

\# GREEN: \{score\_GREEN\} (\{pct\_green\} percentile)\\
\textit{$<$note: The score ranges from 1 to 5, and higher scores indicate a higher affinity for environmentalism.$>$}

\# Big 5 Personality:\\
\textit{$<$note: Openness reflects curiosity and receptiveness to new experiences, Conscientiousness indicates self-discipline and goal-directed behavior, Extraversion measures sociability and assertiveness, Agreeableness reflects compassion and cooperativeness, and Neuroticism captures emotional instability and susceptibility to negative emotions. Each score ranges from 1 to 5, and a higher score indicates a greater display of the associated traits.$>$}\\
- Extraversion: \{score\_extraversion\} (\{pct\_extraversion\} percentile)\\
- Agreeableness: \{score\_agreeableness\} (\{pct\_agreeableness\} percentile)\\
- Conscientiousness: \{wave1\_score\_conscientiousness\} (\{pct\_conscientiousness\} percentile)\\
- Openness: \{score\_openness\} (\{pct\_openness\} percentile)\\
- Neuroticism: \{score\_neuroticism\} (\{pct\_neuroticism\} percentile)
\end{prompt}

\section{Theoretical Framework}

\subsection{Impossibility of Ambiguous Prompting}\label{appendix:theory}

\begin{theorem}\label{theorem}[Impossibility Theorem for Ambiguous Prompting Strategies]
For a given ambiguous prompting strategy, there exists no LLM \( f \) that can correctly answer all questions correctly. 
\end{theorem}

\begin{proof}
Assume that \( s(q_1) = s(q_2) = p \) for two distinct questions \( q_1 \) and \( q_2 \) that have different correct answers $a_1$ and $a_2$, with $a_1 \neq a_2$. If the model can correctly answer $q_1$, such that $f(s(q_1)) = f(p) = a_1$, then the model cannot answer $q_2$ because $f(s(q_2)) = f(p) = a_1 \neq a_2$.
\end{proof}

\subsection{Illustration of the underlying DGP that LLM simulation mimics if prompts are interpreted differently}\label{appendix:dgp_illustration}

Consider a true data-generating process (DGP) described by a linear structural equation:
$$
Y = \beta_M D + U
$$
where $\beta_M$ represents the causal effect of $D$ on $Y$ through a set of realistic mechanism $M$, and $U$ represents unobserved factors. In a well-designed experiment, we have $\text{Cov}(U, D) = 0$, making it possible to identify $\beta_M$ and estimate $P(Y|do(D))$ from data on $(Y, D)$ alone. But under a blinded design with an ambiguous prompt, the LLM does not know it should mimic an experimental environment. That is, the LLM does not have enough information to draw $U$ from the appropriate distribution (literally or figuratively - we do not attempt to model how the LLM makes inference but rather present a conceptual model). If the model is mostly trained on observational data, it is mostly likely going to mimic observational data patterns, where $\text{Cov}(U, D) \neq 0$. As a result, the LLM's responses will reflect confounded relationships rather than the intended causal effects. Figure \ref{fig:simple_dgp} summarizes two possible DGPs mimicked by the LLM under this simple ambiguous prompt.

\begin{figure}[htbp]
\centering
\begin{tikzpicture}[
    node distance=1.5cm,
    every node/.style={draw, minimum size=8mm},
    >=stealth
]
    
    \begin{scope}[xshift=0cm]
        \node (D1) at (0,0) {$D$};
        \node (M1) at (2,0) {$M$};
        \node (Y1) at (4,0) {$Y$};
        \node (U1) at (2,2) {$U$};
        
        \draw[->] (D1) -- (M1);
        \draw[->] (M1) -- (Y1);
        \draw[->] (U1) -- (Y1);
        
        \node[above=0.3cm of U1, draw=none] {(a) DGP for Experimental Data};
    \end{scope}

    \begin{scope}[xshift=8cm]
        \node (D2) at (0,0) {$D$};
        \node (M2) at (2,0) {$M$};
        \node (Y2) at (4,0) {$Y$};
        \node (U2) at (2,2) {$U$};
        
        \draw[->] (D2) -- (M2);
        \draw[->] (M2) -- (Y2);
        \draw[->] (U2) -- (Y2);
        \draw[->] (U2) -- (D2);
        
        \node[above=0.3cm of U2, draw=none] {(b) DGP for Observational Data};
    \end{scope}
\end{tikzpicture}

\caption{Possible DGPs mimicked by LLM simulations under ambiguous simple prompts}
\label{fig:simple_dgp}
\end{figure}

As discussed in Section \ref{sec:control}, one might attempt to address this issue by adding more variables $X$ to the prompt. In principle, if $\text{Cov}(U, D|X) = 0$, then controlling for $X$ can help restore identification. However, simply specifying covariates in the LLM prompt does more than condition on them statistically. Including $X$ may cause the simulation to alter the underlying decision-making mechanisms $M$ and shift toward a new set of mechanisms $M'$ and produce an environment that differs substantially from the original scenario -- the estimated parameter $\beta_{M'}$ may have limited ecological validity, since it no longer reflects the realistic mechanisms $M$ that were originally of interest (as we illustrated in Figure \ref{fig:fixed_competing_last_price}). Figure \ref{fig:complex_dgp} summarizes two possible DGPs mimicked by the LLM under this more complicated ambiguous prompt. 

\begin{figure}[htbp]
\centering
\begin{tikzpicture}[
    node distance=1.5cm,
    every node/.style={draw, minimum size=8mm},
    >=stealth
]
    
    \begin{scope}[xshift=0cm]
        \node (D3) at (0,0) {$D$};
        \node (M3) at (2,0) {$M$};
        \node (Y3) at (4,0) {$Y$};
        \node (U3) at (2,2) {$U$};
        \node (X3) at (1,1) {$X$};
        
        \draw[->] (D3) -- (M3);
        \draw[->] (M3) -- (Y3);
        \draw[->] (U3) -- (Y3);
        \draw[->] (U3) -- (X3);
        \draw[->] (X3) -- (D3);
        
        \node[above=0.3cm of U3, draw=none] {(a) Adding $X$ may address confounding};
    \end{scope}

    \begin{scope}[xshift=8cm]
        \node (D4) at (0,0) {$D$};
        \node (M4) at (2,0) {$M$};
        \node (Y4) at (4,0) {$Y$};
        \node (U4) at (2,2) {$U$};
        \node (X4) at (1,1) {$X$};
        
        \draw[->] (D4) -- (M4);
        \draw[->] (M4) -- (Y4);
        \draw[->] (U4) -- (Y4);
        \draw[->] (U4) -- (X4);
        \draw[->] (X4) -- (D4);
        \draw[->] (X4) -- (M4);
        
        \node[above=0.3cm of U4, draw=none] {(b) Adding $X$ may also sacrifice ecological validity};
    \end{scope}
\end{tikzpicture}
\caption{Potential DGPs mimicked by LLM simulations when controlling for additional variables}
\label{fig:complex_dgp}
\end{figure}

\subsection{Implications for Model Development}\label{sec:fine-tuning}

The impossibility theorem extends to fine-tuning and RAG approaches. Consider experimental data where Coca-Cola prices were randomized while Pepsi prices remained fixed. Fine-tuning on this data with ambiguous prompts cannot guarantee the model learns the causal relationship $\beta_M$. The model may learn the regression coefficient $\hat{b} = \beta_M + \text{Cov}(U,D)/\text{Var}(D)$ while maintaining a prior that $\text{Cov}(U,D) \neq 0$, leading to biased causal inference.

Moreover, even if the model correctly learns independence assumptions for the training scenario ($q_1$), it may fail to generalize appropriately to new scenarios ($q_2$). When asked about cases where Pepsi prices vary and Coca-Cola prices are fixed, the model might incorrectly maintain the independence assumption from training, when correlation with unobserved factors would actually exist. Therefore, while fine-tuning and RAG remain valuable for incorporating domain knowledge, they are complements to rather than substitutes for unambiguous prompt designs. 

\section{Fine-tuning with additional irrelevant data}\label{appendix:amazon_fine_tune}

To test the robustness of different fine-tuning strategy when additional dataset is included, we sample $10,000$ observations from the Amazon dataset \cite{berke2024open} and incorporate them into fine tuning using Prompt \ref{prompt:amazon_prompt}).

\begin{prompt}[Prompt for Amazon data fine-tuning]\label{prompt:amazon_prompt}
\textbf{System:} You, AI, are a customer. Your task is to fill in the blank \_\_\_. Return the completed information without extra text.\\
\textbf{User:} Consider the product category: \{category\}.

Suppose you are shopping online, and you see the product \{product\}.

The product is priced at \$\{price\}. You decide to \_\_\_ [``purchase'' or ``not purchase''].

Return example 1: purchase\\
Return example 2: not purchase
\end{prompt}

\end{document}